\documentclass{article}

\def \TRtitle{Multi-Task Learning Using Neighborhood Kernels}
\def \TRauthors{Niloofar Yousefi$^{1}$, Cong Li$^{2}$ , Mansooreh Mollaghasemi$^{1}$ \\  Georgios Anagnostopoulos$^{3}$ and Michael Georgiopoulos$^{2}$}
\def \TRemails{niloofar.yousefi@ucf.edu, licong1112@gmail.com, Mansooreh.Mollaghasemi@ucf.edu, georgio@fit.edu and michaelg@ucf.edu}
\def \TRaffiliations{$^{1}$Department of Industrial Engineering \& Management Systems, University of Central Florida \\ $^{2}$Department of Electrical Engineering and Computer Science, University of Central Florida \\   $^{3}$Department of Electrical and Computer Engineering, Florida Institute of Technology}

\usepackage[hyphens]{url}

\usepackage{./sty/options}
\usepackage{./sty/packages}
\usepackage{./sty/macros}
\usepackage{enumitem}
\usepackage{times,mathtools}

\allowdisplaybreaks

\newcommand{\citep}{\cite}

\begin{document}

% Make title pages
\maketitle

%%%%%%%%%%%%%%%%%%%%%%%%%%%%%%%%%%%%%%%%%%%%%%%%%%%%%%%%%%%%%%%%%%%%%%%%%%%%%%%%
%%%%%%%%%%%%%%%%%%%%%%%%%%%%%%%%%%%%%%%%%%%%%%%%%%%%%%%%%%%%%%%%%%%%%%%%%%%%%%%%
%%%%%%%%%%%%%%%%%%%%%%%%%%%%%%%%%%%%%%%%%%%%%%%%%%%%%%%%%%%%%%%%%%%%%%%%%%%%%%%%

\begin{abstract}
This paper introduces a new and effective algorithm for learning kernels in a \ac{MTL} setting. Although, we consider a \ac {MTL} scenario here, our approach can be easily applied to standard single task learning, as well. As shown by our empirical results, our algorithm consistently outperforms the traditional kernel learning algorithms such as uniform combination solution, convex combinations of base kernels as well as some kernel alignment-based models, which have been proven to give promising results in the past. We present a Rademacher complexity bound based on which a new \ac{MT-MKL} model is derived. In particular, we propose a \ac {SVM}-regularized model in which, for each task, an optimal kernel is learned based on a neighborhood-defining kernel that is not restricted to be \ac {PSD}. Comparative experimental results are showcased that underline the merits of our neighborhood-defining framework in both classification and regression problems.
\end{abstract}

% Do not change. %
\ifMakeReviewDraft
	\linenumbers
\fi

% Do not change. %
%\vskip 0.5in
%\noindent
%{\bf Keywords:} \TRkeywords
% /////////////////////////////// //%

% Input sections from separate files. Modify these fields as necessary.

% reset all acronyms
\acresetall

%%%%%%%%%%%%%%%%%%%%%%%%%%%%%%%%%%%%%%%%%%%%%%%%%%%%%%%%%%%%%%%%%%%%%%%%%%%%%%%%
%%%%%%%%%%%%%%%%%%%%%%%%%%%%%%%%%%%%%%%%%%%%%%%%%%%%%%%%%%%%%%%%%%%%%%%%%%%%%%%%
%%%%%%%%%%%%%%%%%%%%%%%%%%%%%%%%%%%%%%%%%%%%%%%%%%%%%%%%%%%%%%%%%%%%%%%%%%%%%%%%
\section{Introduction}
\label{sec:Introduction}

%\mycomment{A few comments to Dr. A and Niloofar: \\
%1. The reference list has problem: The papers that are not cited in this paper but are given in the .bib file are shown in the reference list. Need to remove the ones that are not cited in the paper. \\
%2. With this revised Introduction section, we are de-emphasizing our relation with \cite{Liu2009} (the original neighborhood kernel paper). It now looks like we are claiming that ``we are proposing this MTL model, without having anything to do with the other models''. This adds some sense of novelty to this paper, instead of saying ``we are simply extending \cite{Liu2009} to MTL scenario''. Therefore, for the experiment section, I think we should put the multi-task experiments first, then show the single-task experiments, and we should claim that the reason to show single-task experiments is to demonstrate the advantage of our method of choosing the ``neighborhood-defining'' kernel by minimizing the bound: it's a good heuristic to select the kernel even in the single-task case. \\
%3. I put the appendix after the reference list. The appendix can be submitted separately in the next week. Therefore, we have additional space. See if we can fill the space withs some images. After all, we are submitting to CVPR, not NIPS.}
%
%\gcacomment{I concur with Cong on all points.}
As shown by the past empirical works \cite{Caruana1997,Ando2005,Argyriou2008,yousefi2015multi,Kumar2012}, it is beneficial to learn multiple related tasks simultaneously instead of independently as typically done in practice. A commonly utilized information sharing strategy for \ac {MTL} is to use a (partially) common feature mapping $\phi$ to map the data from all tasks to a (partially) shared feature space $\mathcal{H}$. Such a method, named kernel-based \ac{MTL}, not only allows information sharing across tasks, but also enjoys the non-linearity that is brought by the feature mapping $\phi$.

While applying kernel-based models, it is crucial to carefully choose the kernel function, as using inappropriate kernel functions may lead to deteriorated generalization performance. A widely adapted strategy for kernel selection is to learn a convex combination of some base kernels \cite{Kloft2011,Lanckriet2004}, which combined with \ac{MTL}, results in the \ac{MT-MKL} approach. Such a method linearly combines $M$ pre-selected basis kernel functions $k_1, \cdots, k_M$, with the combination coefficients $\boldsymbol{\theta} := [\theta_1, \cdots, \theta_M]$, which are learned during the training stage in a pre-defined feasible region. For example, a widely used and theoretically well studied feasible region is given by the $L_p$-norm constraint \cite{Kloft2011}: $\Psi(\boldsymbol{\theta}) := \{ \boldsymbol{\theta} : \boldsymbol{\theta} \succeq \boldsymbol{0}, \| \boldsymbol{\theta} \|_p \leq 1 \}$. As such, each task features a common kernel function $k := \sum_{m=1}^{M} \theta_m k_m$. One such \ac{MT-MKL} model is proposed in \cite{Samek2011}. Besides, a more general \ac{MT-MKL} approach with conically combined multiple objective functions and $L_p$-norm \ac{MKL} constraint is introduced in \cite{Li2014b}, and further extended and theoretically studied in \cite{Li2014c}. A \ac{MT-MKL} model that allows both feature and kernel selection is proposed in \cite{Jebara2004} and extended in \cite{Jebara2011}. Finally, in \cite{Tang2009}, the authors proposed to use a partially shared kernel function, \ie, $k_t := \sum_{m=1}^{M} (\mu^m + \lambda_t^m) k_m$, with $L_1$-norm constraints are put on $\boldsymbol{\mu}$ and $\boldsymbol{\lambda}$. Such a method allows data from the unrelated tasks to be mapped to task-specific feature spaces, instead of sharing feature space with other tasks, thus potentially prevents the effect of ``negative transfer'', \ie, knowledge transferred between irrelevant tasks, which leads to degraded generalization performance.

Another rather different approach for learning kernels is based on the notion of \ac {KTA}, which is a similarity measure between the input and output (target) kernels. There exist several studies that utilize the kernel alignment \cite{cristianini2001kernel,gretton2005measuring,kim2006optimal} or centered kernel alignment \cite{cortes2012algorithms} as their kernel learning criteria. It has been theoretically shown that maximizing the alignment between the input kernel and the target one can lead to a highly accurate learning hypothesis (see Theorem 13 for classification and Theorem 14 for regression in \cite{cortes2012algorithms} ). Also, via a series of experiments, the authors in \cite{cortes2012algorithms} demonstrate that the alignment approach consistently outperforms the traditional kernel-based methods such as uniform or convex combination solutions. As shown in \cite{cortes2012algorithms}, the problem of learning a maximumly aligned kernel to the target can be efficiently reduced to a simple \ac {QP}, which in turn is equivalent to considering a Frobenius norm of differences between the input and target kernels.

Inspired by the idea of kernel alignment, in this paper we present a new \ac {MT-MKL} model in which, for each task $t$, the ``optimal" kernel matrix $\boldsymbol{K}_{t}$ is highly aligned with a ``neighborhood-defining" alignment matrix $\boldsymbol{\hat{K}}_{t}$, which is dictated by the data itself. In particular, we derive a Rademacher complexity bound for \ac {MTL} function classes induced by an alignment-based regularization. It turns out that the Rademacher complexity of such classes can be upper-bounded in terms of the neighborhood alignment matrices. Based on this observation, we derive a new algorithm where the optimal kernels are learned simultaneously with the alignment matrices, using a regularized \ac {SVM} optimization problem. As opposed to the target kernel alignment approach (in which the alignment kernel $\boldsymbol{y}\boldsymbol{y}'$ is \ac {PSD}) , we do not restrict our alignment matrices to be \ac {PSD}. Therefore, our model enjoys more flexibility in the sense that it allows the optimal kernel to reside in the neighborhood of an indefinite matrix, whenever warranted by the data. 

It is worth pointing out that the problem of learning with indefinite kernels has been addressed by many researchers \cite{ong2004learning,kundu2010efficient,zhuang2011family,gu2012learning,loosli2016learning}, as it has been shown that in many real-life applications, the \ac {PSD}-ness constraints on the kernels might limit the usability of kernel-based methods. (see \cite{ong2004learning} for a discussion). Examples of such situations include using the BLAST, Smith-Waterman or FASTA similarity scores between protein sequences in bioinformatics; using the cosine similarity between term frequency-inverse document frequency (tf-idf) vectors in text mining; using the pyramid match kernel, shape matching and tangent distances in computer vision; using human-judged similarities between concepts and words in information retrieval, just to name a few.

Finally, via a series of comparative experiments, we show that our proposed model surpasses in performance the traditional kernel-based learning algorithms such as uniform and convex combination solutions. As shown by the experiments, our method also improves upon the \ac {KTA} approach in which an optimal kernel is learned by maximizing the alignment between the target kernel $\boldsymbol{\hat{K}} = \boldsymbol{y} \boldsymbol{y}'$ and the convex contamination kernel $\boldsymbol{K} = \sum_{m=1}^{M} \theta_{m} \boldsymbol{K}_{m}$. Moreover, we show that our model empirically outperforms some other similar approaches of learning an optimal neighborhood kernel. However, as we discuss later, the similarity between our model and other optimal neighborhood kernel learning models \cite{Liu2009,Liu2013} is only superficial. 

The remainder of this paper is organized as follows: \sref{sec:NewModel} contains a formal description of a \ac {MTL} alignment-regularized learning framework with fixed alignment matrices. \sref{sec:GeneralizationBounds} presents a Rademacher complexity bound for the corresponding hypothesis class of alignment-based models presented in \sref{sec:NewModel}. Also in \sref{sec:NeighKernelChoice}, we present our new \ac {MTL} model, and further discuss the motivation and subsequent derivation of optimal neighborhood-defining kernels. Experimental results obtained for \ac{MT} classification and regression are provided in \sref{sec:Experiments} to show the effectiveness of the proposed model compared to some other kernel-based formulations. Finally, in \sref{sec:Conclusions} we briefly summarize our findings.  

In what follows, we use the following notational conventions: vectors and matrices are depicted in bold face. A prime $'$ denotes vector/matrix transposition. The ordering symbols $\succeq$ and $\preceq$ stand for the corresponding component-wise relations. Additional notation is defined in the text as needed.
%
%GCA notes: As a constraint, we employ a sum of Frobenius norms of differences between input and target kernels for each task motivated by Kernel Target Alignment generalization guarantees, which we have extended to the case of \ac{MTL} in Proposition X.   

%%%%%%%%%%%%%%%%%%%%%%%%%%%%%%%%%%%%%%%%%%%%%%%%%%%%%%%%%%%%%%%%%%%%%%%%%%%%%%%%
%%%%%%%%%%%%%%%%%%%%%%%%%%%%%%%%%%%%%%%%%%%%%%%%%%%%%%%%%%%%%%%%%%%%%%%%%%%%%%%%
%%%%%%%%%%%%%%%%%%%%%%%%%%%%%%%%%%%%%%%%%%%%%%%%%%%%%%%%%%%%%%%%%%%%%%%%%%%%%%%%
\section{Multi-Task with Neighborhood-defining Matrices}
\label{sec:NewModel}
Consider a linear \ac{MTL} model involving $T$ tasks, each of which is addressed by an \ac{SVM} model. For each supervised task $t$, assume that there is a training set $\left\{ \left( x_t^i, y_t^i \right) \right\}_{i=1}^{n}$ sampled from  $\mathcal{X} \times \mathcal{Y}$ based on some probability distribution $P_t(x, y)$, where $\mathcal{X}$ denotes an arbitrary set that serves as the native space of samples for all tasks and $\mathcal{Y}$ represents the output space associated with the labels. Without loss of generality, we assume that the same number $n$ of labeled samples are available for learning each task. Furthermore, we assume that the $T$ \ac{SVM} tasks are going to be learned via a standard \ac{MKL} scheme using a prescribed collection of \acp{RKHS} $\left\{ \mathcal{H}_m \right\}_{m=1}^M$, such that each $\mathcal{H}_m$ is equipped with an inner product $\left\langle \cdot,\cdot \right\rangle_{\mathcal{H}_m}$ and that has an associated feature mapping $\phi_m: \mathcal{X} \rightarrow \mathcal{H}_m$. The associated reproducing kernel $k_m: \mathcal{X} \times \mathcal{X} \rightarrow \mathbb{R}$ is such that $k_m(x_1,x_2) = \left\langle \phi_m(x_1), \phi_m(x_2)  \right\rangle_{\mathcal{H}_m}$ for all $x_1, x_2 \in \mathcal{X}$. 

It is not hard to verify that this considerations can imply an equivalent \ac{RKHS} $\mathcal{H}_t$ that serves as the partially common feature space for all $T$ tasks. In specific, one can consider $\mathcal{H}_t = \bigoplus_{m=1}^M \sqrt{\theta_t^m} \mathcal{H}_m$ with induced feature mapping $\phi_t := [ \sqrt{\theta_t^1}{\phi_1}' \cdots $ $ \sqrt{\theta_t^M} {\phi_M}']'$, endowed with the inner product $\left\langle \cdot, \cdot \right\rangle_{\mathcal{H}_t} = \sum_{m=1}^M \theta_t^m \left\langle \cdot, \cdot \right\rangle_{\mathcal{H}_m}$ and its associated reproducing kernel function $k_{t}(x_1,x_2) = \sum_{m=1}^M \theta_t^m k_m(x_1,x_2)$ for all $x_1, x_2 \in \mathcal{X}$. Define $\boldsymbol{f}:\mathcal{X} \rightarrow \mathbb{R}^T$ as $\boldsymbol{f}:= \left[ f_1 \ldots f_T \right]'$, where $f_t(x) := \left\langle  w_t, \phi_t(x) \right\rangle_{\mathcal{H}_t} + b_t$ for all $t$. Also let $\boldsymbol{\theta}_{t}:=(\theta_{t}^{1},\ldots,\theta_{t}^{M}) \in \mathbb{R}^{M}$. At this point, we would like to bring into attention that in order to address the problem of negative transfer, for each task $t$, we let  $\boldsymbol{\theta}_{t} := \boldsymbol{\mu} + \boldsymbol{\lambda}_{t}$, where $\boldsymbol{\mu} := [\mu^1, \cdots, \mu^M]' \in \mathbb{R}^M$, $\boldsymbol{\lambda}_{t} := [\lambda_1^1, \cdots, \lambda_t^M]' \in \mathbb{R}^{M}$. We also define $\boldsymbol{\theta} := [\boldsymbol{\mu}', \boldsymbol{\lambda}_{1}',\ldots, \boldsymbol{\lambda}_{T}]' \in \mathbb{R}^{M + MT}$, which is the concatenation of the mutual vector $\boldsymbol{\mu}$ and the task-specific vector parameters $\boldsymbol{\lambda}_{t}$s. Then, we consider the following Hypothesis class
 \begin{align}
 \label{MTL-ONK-HS1}
 \mathcal{F} := \left\lbrace x \rightarrow \boldsymbol{f}(x) : \forall t,\; \boldsymbol{w}_{t} \in \mathcal{H}_{t}, \sum_{t=1}^{T} \| \boldsymbol{w}_{t} \|_{\mathcal{H}_{t}}^2 \leq R, \boldsymbol{\theta} \in \Psi_{\boldsymbol{\theta}} \right\rbrace % \boldsymbol{\hat{K}}_{t} \in \Omega(\boldsymbol{\hat{K}}_{1},\ldots,\boldsymbol{\hat{K}}_{T})  
 \end{align}
%\begin{align}
%\label{MTL-ONK-Primal}
%  \min_{ \boldsymbol{w},\boldsymbol{b},\boldsymbol{\theta} } \frac{1}{2} \sum_{t=1}^{T} \| \boldsymbol{w}_{t} \|_{\mathcal{H}_{t}}^{2} + C \sum_{t=1}^{T} \sum_{i=1}^{n} \ell \left(f_{t}(x_{t}^{i}),y_{t}^{i} \right)  
% + \frac{\eta}{2} \sum_{t=1}^{T} \| \boldsymbol{K}_{t} - \boldsymbol{\hat{K}}_{t} \|_{F}^{2} % + \Omega(\boldsymbol{\hat{K}}_{1},\ldots,\boldsymbol{\hat{K}}_{T})
%\end{align}
where  $\Psi_{\theta} :=  \left\{ \boldsymbol{\theta} \in \mathbb{R}^{M+MT}  \ : \ \boldsymbol{\theta} \succeq \mathbf{0}, \sum_{t=1}^{T} \| \boldsymbol{K}_{t} - \boldsymbol{\hat{K}}_{t} \|_{F}^{2} \leq \rho \right\}$, $\boldsymbol{K}_t := \sum_{m = 1}^{M} \theta_t^m \boldsymbol{K}_t^m$, with $\boldsymbol{K}_t^m \in \mathbb{R}^{n \times n}$ being the kernel matrix, whose $(i,j)$ entry is given as $k_m(x_t^i, x_t^j)$. Also, $\boldsymbol{\hat{K}}_{1},\ldots,\boldsymbol{\hat{K}}_{T}$ are $T$ neighborhood-defining matrices, which are assumed to be pre-defined at this moment. We will show in \sref{sec:NeighKernelChoice} how these matrices can be determined based on the Rademacher complexity of the model. Note that via Tikhonov-Ivanov equivalency, one can show that the Frobenius norm constraint in the set $\Psi_{\theta}$ can be equivalently considered as a regularization term in the objective function of the corresponding learning problem of hypothesis class \eqref{MTL-ONK-HS1}. Furthermore, it can be shown that minimizing over this term can be reduced to a simple \ac {QP} itself, which in turn, is equivalent to an alignment maximization problem between two kernels $\boldsymbol{K}_{t}$ and $\boldsymbol{\hat{K}}_{t}$ (see Proposition 9 in \cite{cortes2012algorithms}).

With this being said, if one defines $\boldsymbol{\hat{K}}_{t}:=\boldsymbol{y}_{t} \boldsymbol{y}'_{t}$ for each task $t$, then the term $\| \boldsymbol{K}_{t} - \boldsymbol{\hat{K}}_{t} \|_{F}^{2}$ in \eqref{MTL-ONK-HS1} reduces to the \ac {KTA} quantity, which measures the alignment between the kernel $\boldsymbol{K}_{t}$ and target kernel matrix $\boldsymbol{y}_{t} \boldsymbol{y}'_{t}$, derived from the output labels. Obviously, unlike the idea of this paper in which the neighborhood-defining matrix is also learned in a data-driven manner, the target kernel $\boldsymbol{y}_{t} \boldsymbol{y}'_{t}$ is fixed.

Other approaches in single task context \cite{Liu2009,Liu2013} exist that also consider the problem of learning an optimal kernel from a noisy observation. However, these approaches are different in spirit from our approach here. These differences can be summarized as follows: (1) they assume that both the optimal kernel and the noisy one are \ac {PSD} matrices, (2) they use the neighborhood defining kernel during the training, and the proxy kernel during the test procedure, and therefore past and future examples are treated inconsistently by their model, and (3) more importantly, in both approaches, the feature space is assumed to be induced by the neighborhood-defining kernel $\boldsymbol{\hat{K}}$, (and not the original kernel $\boldsymbol{K}$). One potential reason for this consideration might be related to the fact that assuming a feature space induced by the kernel $\boldsymbol{K}$ leads to the trivial solution $\boldsymbol{\hat{K}} \equiv \boldsymbol{K}$ in their formulations. 
%
%
%these approaches either (rather arbitrarily) pre-specify the neighborhood-defining matrix or learn it within a pre-defined, potentially limited, feasible region to avoid the trivial solution of $\boldsymbol{\hat{K}} \equiv \boldsymbol{K}$. 
%
% Note that, if one defines $\boldsymbol{\hat{K}}_{t}:=\boldsymbol{y}_{t} \boldsymbol{y}'_{t}$ for each task $t$, then the term $\| \boldsymbol{K}_{t} - \boldsymbol{\hat{K}}_{t} \|_{F}^{2}$ in \eqref{MTL-ONK-Primal}, is closely related to the notion of \ac {KTA}, which is a similarity measure between kernel $\boldsymbol{K}_{t}$ and the target kernel matrix $\boldsymbol{y}_{t} \boldsymbol{y}'_{t}$ derived from the output labels. Note that, unlike the idea of this paper in which the neighborhood-defining matrix is assumed to be indefinite, the target matrix $\boldsymbol{y}_{t} \boldsymbol{y}'_{t}$ is \ac{PSD}.
%

In the next section, we present Rademacher complexity bound for the hypothesis class in \eqref{MTL-ONK-HS1}, which helps us in designing a new \ac {MTL} model with a regularization term on $\boldsymbol{\hat{K}}_{1},\ldots,\boldsymbol{\hat{K}}_{T}$ based on the complexity of the model.

\section{Rademacher Complexity}
\label{sec:GeneralizationBounds}
Rademacher complexity is a measure of how well the functions in a hypothesis class $\mathcal{G}$ correlates with the random noise, and therefore it quantifies the richness of the hypothesis set $\mathcal{G}$. Given a space $\mathcal{Z}$, let $S = \{ z_{1},\ldots z_{n} \}$ be a set of data, which are drawn identically and independently according to distribution $P$. Then, the \textit{Empirical Rademacher Complexity} of the hypothesis class $ \mathcal{G}:= \{ g: \mathcal{Z} \rightarrow \mathbb{R} \}$ is defined as
 \begin{align}
 \label{RC}
 \mathfrak{\hat{R}} (\mathcal{G}) : =  \mathbb{E}_{\sigma} \left\lbrace \sup_{g \in \mathcal{G}} \frac{1}{n} \sum_{i=1}^{n} \sigma_{i} g(z_{i}) \right\rbrace 
 \end{align}
 where $\sigma_{i}$s are independent uniform $\{\pm 1\}$-valued random variables. Rademacher complexities are data-dependent complexity measures that lead to finer learning guarantees \cite{koltchinskii2002empirical,bartlett2002rademacher}.

With some algebra manipulation over the term $\sum_{t=1}^{T} \| \boldsymbol{K}_{t} - \boldsymbol{\hat{K}}_{t} \|_{F}^{2}$ in \eqref{MTL-ONK-HS1}, it is not difficult to see that the constraint set $\Psi_{\boldsymbol{\theta}}$ for $\boldsymbol{\theta}$ is obtained as
%It is not difficult to see that the \ac{HS} corresponding to problem \eqref{MTL-ONK-Primal} can be given as
%\begin{align}
%\label{MTL-ONK-HS}
%\mathcal{F} := \left\lbrace x \rightarrow \boldsymbol{f}(x) : \forall t,\; \boldsymbol{w}_{t} \in \mathcal{H}_{t}, \sum_{t=1}^{T} \| \boldsymbol{w}_{t} \|_{\mathcal{H}_{t}}^2 \leq R, \boldsymbol{\theta} \in \Psi_{\boldsymbol{\theta}}\right\rbrace % \boldsymbol{\hat{K}}_{t} \in \Omega(\boldsymbol{\hat{K}}_{1},\ldots,\boldsymbol{\hat{K}}_{T})  
%\end{align}
%where,  the constraint set $\Psi_{\boldsymbol{\theta}}$ is obtained as
\begin{align}
\label{eq:PsiTheta}
	\Psi_{\theta} := & \left\{ \boldsymbol{\theta} \in \mathbb{R}^{M+MT}  \ : \ \boldsymbol{\theta} \succeq \mathbf{0}, \   \boldsymbol{\theta}' \boldsymbol{A} \boldsymbol{\theta} -  \boldsymbol{\theta}' \boldsymbol{b} + c \leq \rho \right\}
\end{align}
where the definition of $\boldsymbol{A}$, $\boldsymbol{b}$ and $c$ are given as follows: 

First, $\boldsymbol{A}$ is a block matrix that is defined as
\begin{align}
\label{eq:blockAMatrixDef}
\boldsymbol{A} := \begin{bmatrix}
\boldsymbol{A}_{1} & \boldsymbol{A}_{3}\\ 
\boldsymbol{A}_{3}' & \boldsymbol{A}_{2} 
\end{bmatrix} \in \mathbb{R}^{(M+MT) \times (M+MT)}
\end{align}
Here, $\boldsymbol{A}_{1} \in \mathbb{R}^{M \times M}$, whose $(m_1, m_2)$-th element is given as $\trace{ \sum_{t=1}^{T} \boldsymbol{K}_t^{m_1} \boldsymbol{K}_t^{m_2} }$. $\boldsymbol{A}_{3} \in \mathbb{R}^{M \times MT}$ is a $M \times M$ block matrix, whose $(m_1, m_2)$-th block is a $T$-dimensional vector $\boldsymbol{a}^{(m_1, m_2)} := \left[\trace{\boldsymbol{K}_1^{m_1} \boldsymbol{K}_1^{m_2}} \cdots \trace{\boldsymbol{K}_T^{m_1} \boldsymbol{K}_T^{m_2}} \right]' \in \mathbb{R}^T$. Similar to $\boldsymbol{A}_3$, $\boldsymbol{A}_2 \in \mathbb{R}^{MT \times MT}$ is also a $M \times M$ block matrix, whose $(m_1, m_2)$-th block is a $T \times T$ diagonal matrix, where the $t$-th diagonal element is given by $a_t^{(m_1, m_2)}$, \ie, the $(m_1, m_2)$-th block matrix of $\boldsymbol{A}_2$ is defined as $\textit{diag}(\boldsymbol{a}^{(m_1, m_2)})$. Note that it can be easily shown that the matrix $\boldsymbol{A}$ is \ac{PD} whenever the base kernels ($\boldsymbol{K}_{m}$s) are linearly independent. We assume without loss of generality that matrix $\boldsymbol{A}$ is \ac {PD}, otherwise we can choose an independent subset of base kernels.

 Second, given a matrix $\boldsymbol{B} \in \mathbb{R}^{M \times T}$, whose $(m, t)$-th element is defined as $b_t^m := 2 \trace{\boldsymbol{K}_t^m \hat{\boldsymbol{K}_t}}$, the vector $\boldsymbol{b} \in \mathbb{R}^{M + MT}$ in (\ref{eq:PsiTheta}) is given as 
\begin{align}
\label{eq:blockbVectorDef}
\boldsymbol{b} := \begin{bmatrix}
\boldsymbol{B} \boldsymbol{1}_T \\ 
\vec{\boldsymbol{B'}}  
\end{bmatrix}
\end{align}
where $\boldsymbol{1}_T$ is the $T$-dimensional all-one vector, and $\vec{*}$ is the matrix vectorization operator.

 Finally, $c$ in (\ref{eq:PsiTheta}) is defined as
\begin{align}
\label{eq:cDef}
c := \trace{ \sum_{t=1}^{T} \hat{\boldsymbol{K}}_t \hat{\boldsymbol{K}}_t }
\end{align}
%
%
%Assume that $(x_t, y_t)$ are random variables with probability distribution $P_t(x, y)$. Let $er: \mathcal{F} \rightarrow \left[0, 1\right]$ be the  multi-task misclassification error defined as 
%\begin{align}
%	\label{eq:error}
%	er\left( f \right) := \frac{1}{T} \sum_{t=1}^T  \E{ \left[ y_t f_t(x_t) < 0 \right] }
%\end{align}
%\noindent
%where $\left[ \cdot \right]$ denotes the Iverson bracket. Note that $\left[ \mathrm{predicate} \right] = 1$, if $\mathrm{predicate}$ is true, and equals $0$, if otherwise. Furthermore, let $\widehat{er}_{\gamma}: \mathcal{X} \rightarrow \left[0, 1\right]$, where $\gamma > 0$, denote the empirical multi-task margin error defined as
%\begin{align}
%	\label{eq:merror}
%	\widehat{er}_{\gamma}\left( f \right) & := \frac{1}{TN} \sum_{t=1}^T \sum_{i=1}^n  \Lambda_{\gamma} \left( y_t^i f_t(x_t^i) \right)
%	\intertext{where}
%	\Lambda_{\gamma}(u) & :=
%	\begin{cases}
%		1 & \text{if} \ u<0 \\
%		1-\frac{u}{\gamma} & \text{if} \ 0 \leq u \leq \gamma \\
%		0 & \text{if} \ \gamma < u
%	\end{cases}
%\end{align}
%\noindent
%is the margin loss function. Then, it can be shown that the following result holds:
%\mycomment{In the following theorem, you need to show that the matrix $\boldsymbol{A}$ is invertable, otherwise \erefequ{eq:Rub} is not well defined.}
\begin{thrm}
\label{thrm:thrm}
 Let $\mathcal{F}$ be the \ac{HS} defined in \eref{MTL-ONK-HS1}. Then, for all $f \in \mathcal{F}$ and fixed neighborhood-defining matrices $\left\{ \hat{\boldsymbol{K}}_t \right\}_{t=1}^T$, it holds that the empirical Rademacher complexity $\hat{\mathfrak{R}}(\mathcal{F})$ can be upper-bounded as
%\begin{align}
%\label{eq:thmBound}
%	er\left( f \right) \leq \widehat{er}_{\gamma}\left( f \right) + \frac{2}{\gamma} \hat{R}\left( \mathcal{F} \right) + \sqrt{ \frac{9 \ln \frac{2}{\delta}}{2 T N}}
%\end{align}
%where
\begin{align}
\label{eq:Rub}
		\hat{\mathfrak{R}}(\mathcal{F}) \leq & \frac{1}{n} \sqrt{ \frac{R}{2 T} }  
		\left( \boldsymbol{d}' \boldsymbol{A}^{-1} \boldsymbol{b}  + \frac{1} {2} \left[\left( \boldsymbol{d}' \boldsymbol{A}^{-1} \boldsymbol{d} + 2 \trace{\boldsymbol{V} \boldsymbol{A}^{-1} \boldsymbol{V}'} \right) + \left( \boldsymbol{b}' \boldsymbol{A}^{-1} \boldsymbol{b} + 4(\rho - c) \right) \right] \right)^{\frac{1}{2}} 
\end{align}
\noindent
with 
\begin{align}
\label{eq:blockdVectorDef}
\boldsymbol{d} := \begin{bmatrix}
\boldsymbol{D} \boldsymbol{1}_T \\ 
\vec{\boldsymbol{D'}}  
\end{bmatrix} \in \mathbb{R}^{M+MT}, \quad \quad \boldsymbol{V} : = \begin{bmatrix} \boldsymbol{V}_{1}\\ \vdots\\ \boldsymbol{V}_{T}\end{bmatrix} \in \mathbb{R}^{Tn^{2} \times (M+MT)}
\end{align}
\noindent
where $	\boldsymbol{V}_t := \left[ \tilde{\boldsymbol{V}}_t \:\: \tilde{\boldsymbol{V}}_t \bigotimes \boldsymbol{e}_t'
	\right] \in \mathbb{R}^{n^2 \times (M+MT)}$, $\tilde{\boldsymbol{V}}_t := \left[\vec{\boldsymbol{K}_t^1} \cdots \vec{\boldsymbol{K}_t^M} \right] \in \mathbb{R}^{n^2 \times M}$, $\boldsymbol{e}_t \in \mathbb{R}^T$ is a vector whose $t$-th element is $1$ and other elements are $0$, and $\bigotimes$ stands for the Kronecker product. Also, $\boldsymbol{D} \in \mathbb{R}^{M \times T}$ is a matrix whose $(m, t)$-th element is defined as $d_t^m := \trace{\boldsymbol{K}_t^m}$.
\end{thrm}
\noindent
The proof of this theorem is provided in the Appendix.

In the next section, we present a new \ac{MT-ONMKL} formulation, which enjoys a data-driven procedure for selecting the optimal kernel from neighborhood alignment matrices. More specifically, our model learns the optimal kernel $\boldsymbol{K}$ by considering an additional regularization term derived based on the Rademacher complexity of the alignment-regularized hypothesis space \eqref{MTL-ONK-HS1}.

\section{The New \ac{MT-ONMKL} model}
\label{sec:NeighKernelChoice}
Since the Rademacher complexity bounds give guarantees for the generalization error, they are considered as one of the most helpful data-dependent complexity measures in both theoretical analysis and designing of more efficient algorithms in machine learning problems. As an example, in the context of kernel-based learning methods, most algorithm restrict the learning hypothesis class to a constraint on the trace of the kernel, as it has been shown that, for a fixed kernel, the Rademacher complexity of a kernel-based algorithm can be upper-bounded in terms of the trace of the kernel \cite{bach2004multiple,Lanckriet2004,Sonnenburg2006}. Here, we also derive an upper bound for Rademacher complexity of alignment-regularized hypothesis classes, similar to \eqref{MTL-ONK-HS1}, and then we design a new \ac {MT-MKL} based on our derived bound. 
\subsection{Formulation}
As we showed earlier, in terms of the neighborhood-defining matrices $\boldsymbol{\hat{K}}_t$s, the Rademacher complexity of $\mathcal{F}$ in \eqref{MTL-ONK-HS1} can be upper-bounded by the quantity $\boldsymbol{d}' \boldsymbol{A}^{-1} \boldsymbol{b} + \frac{1}{2} \boldsymbol{b}' \boldsymbol{A}^{-1} \boldsymbol{b} -4c$. Thus, we can add this constraint to restrict the hypothesis class $\mathcal{F}$, which leads to a new \ac {MTL} formulation presented in the sequel. In particular, considering that part of \eqref{eq:Rub} which depends on $\boldsymbol{\hat{K}}_{t}$s, we define the following regularizer to learn our neighborhood-defining measures $\boldsymbol{\hat{K}}_{t}$s
\begin{align}
\label{Reg-K_hat}
\Omega(\boldsymbol{\hat{K}}):=  \boldsymbol{d}' \boldsymbol{A}^{-1} \boldsymbol{b} + \frac{1}{2} \boldsymbol{b}' \boldsymbol{A}^{-1} \boldsymbol{b} -4c  
\end{align}
%Note that with some effort, it can be shown that the above function is convex in the neighborhood-defining kernels $\boldsymbol{\hat{K}}_{t}$s, if one lets $\boldsymbol{\hat{K}_{t}}$s belong to the set $\Psi_{\boldsymbol{\hat{K}}} := \{ \boldsymbol{\hat{K}} \in \boldsymbol{S}^{n}_{+} \mid  \trace{\boldsymbol{\hat{K}}} = \tau \}$, where $\boldsymbol{S}^{n}_{+}$ is the set of positive semidefinite $n \times n$ matrices, and $\tau > 0$ is a given constant. 
where $\boldsymbol{\hat{K}}:=(\boldsymbol{\hat{K}}_{1},\ldots,\boldsymbol{\hat{K}}_{T})$. Now, we formulate our new \ac {MT-ONMKL} model as the following optimization problem
%  \in \Psi_{\boldsymbol{\hat{K}}}
\begin{align}
\label{New-MT-ONMKL}
  \min_{\boldsymbol{\theta} \succeq 0, \boldsymbol{\hat{K}}} \min_{\boldsymbol{w},\boldsymbol{b}}\;  \frac{1}{2} \sum_{t=1}^{T} \| \boldsymbol{w}_{t} \|_{\mathcal{H}_{t}}^{2} + C \sum_{t=1}^{T} \sum_{i=1}^{n}  \left(f_{t}(x_{t}^{i}),y_{t}^{i} \right)  
 + \frac{\eta}{2} \sum_{t=1}^{T} \| \boldsymbol{K}_{t} - \boldsymbol{\hat{K}}_{t} \|_{F}^{2}  + \frac{\beta}{2} \Omega(\boldsymbol{\hat{K}})
\end{align}
where $\boldsymbol{w}:=(\boldsymbol{w}_{1},\ldots,\boldsymbol{w}_{T})$, $\boldsymbol{b} := (b_{1},\ldots,b_{T})$; $\eta$ and $\beta$ are regularization parameters.

Unlike approaches such in \cite{Liu2013} and \cite{Liu2009}, \ac{MT-ONMKL} opts to choose the neighborhood-defining matrices $\boldsymbol{\hat{K}}_{t}$ using optimization problem \eqref{New-MT-ONMKL}, in lieu of choices that are much harder to justify. The benefits of this particular choice are largely reflected in the experimental results reported in the next section.

%It can be observed that, our method provides a closed-form solution to the pre-specified kernel matrix, thus is efficient to calculate. This is contrary to the method in \cite{Liu2013}, as stated above. Moreover, unlike \cite{Liu2009}, our method is based on an appealing intuition, \ie, selecting the pre-specified kernel that minimizes the \ac{ERC} bound, instead of randomly choosing it.

\subsection{Algorithm}
\label{sec:Alg}
First, note that if one considers $T$ inter-related \ac{SVM} classification problems, then \eqref{New-MT-ONMKL} can be equivalently expressed as 
%\in \Psi_{\boldsymbol{\hat{K}}}
\begin{align}
	\label{eq:primalFormulationEquivalent}
	 \min_{\boldsymbol{\theta} \succeq 0, \boldsymbol{\hat{K}}}  \min_{\boldsymbol{w}, \boldsymbol{b}, \boldsymbol{\xi}} \;\; & \frac{1}{2} \sum_{t=1}^{T} \| w_t \|^2_{\mathcal{H}_t}  +  C \sum_{t=1}^{T}\sum_{i=1}^{n} \xi_t^i
	 + \frac{\eta}{2} \sum_{t=1}^T \left\| \boldsymbol{K}_t - \hat{\boldsymbol{K}}_t \right\|^2_{F}  +  \frac{\beta}{2} \Omega(\boldsymbol{\hat{K}})
	\nonumber\\
     \textit{s.t.} \quad & \forall i,\; 
   y_t^i \left( \left\langle w_t, \phi_t\left( x_t^i \right) \right\rangle_{\mathcal{H}_t} + b_t\right)   \geq 1 - \xi_t^i, \: \xi_t^i  \geq 0
\end{align}
where $\boldsymbol{\xi} := (\boldsymbol{\xi}_{1},\ldots,\boldsymbol{\xi}_{T})$. Note that a very similar formulation can be derived for regression using algorithms such as SVR at this stage. Thus, the algorithm that we present in the following can be easily extended to regression problems with a simple substitution of SVM with SVR.

It can be shown that the primal-dual form of \eqref{eq:primalFormulationEquivalent} with respect to $\{ \boldsymbol{\hat{K}},\boldsymbol{\theta} \}$ and $\left\lbrace\boldsymbol{w}, \boldsymbol{b}, \boldsymbol{\xi} \right\rbrace$ is given by
\begin{align}
\label{eq:dualFormulationEquivalent}
     \min_{\boldsymbol{\theta} \succeq 0, \boldsymbol{\hat{K}}}  \max_{\{ \boldsymbol{\alpha}_{t}\}_{t}} \;\; & \sum_{t=1}^{T} \boldsymbol{\alpha}_{t}' \boldsymbol{1}_{n}  - \frac{1}{2} \sum_{t=1}^{T} \boldsymbol{\alpha}_{t}' \boldsymbol{Y}_{t} \boldsymbol{K}_{t} \boldsymbol{Y}_{t} \boldsymbol{\alpha}_{t}
	+ \frac{\eta}{2} \sum_{t=1}^T \left\| \boldsymbol{K}_t - \hat{\boldsymbol{K}}_t \right\|^2_{F}  + \frac{\beta}{2} \Omega(\boldsymbol{\hat{K}})
	\nonumber\\
	 \textit{s.t.} \quad & \forall t,\; 
    \boldsymbol{0} \preceq \boldsymbol{\alpha}_{t}	\preceq  C \boldsymbol{1} , \; \boldsymbol{y}_{t}' \boldsymbol{\alpha}_{t} = 0
\end{align}
where $\boldsymbol{\alpha}_{t}$ is the Lagrangian dual variable for the minimization problem \wrt $\{\boldsymbol{w}_{t}, \boldsymbol{b}_{t}, \boldsymbol{\xi}_{t}\}$. A block coordinate descent framework can be applied to decompose \probref{eq:dualFormulationEquivalent} into three subproblems. The first subproblem which is the maximization problem with respect to $\boldsymbol{\alpha}$, can be efficiently solved via \texttt{LIBSVM} \cite{CC01a}, and the second subproblem, which is the minimization with respect to $\boldsymbol{\theta}$, takes the quadratic form 
\begin{align}
\label{Optimize-Theta} 
\min_{\boldsymbol{\theta} \succeq 0} \; \boldsymbol{\theta}' \boldsymbol{A} \boldsymbol{\theta} -  \boldsymbol{\theta}' \left(  \mathbf{b} + \boldsymbol{q} \right) 
\end{align}
where the vector $\boldsymbol{q} \in \mathbb{R}^{M + MT}$ is given as 
\begin{align}
\label{eq:blockbVectorDef-q}
\boldsymbol{q} := \begin{bmatrix}
\boldsymbol{Q} \boldsymbol{1}_T \\ 
\vec{\boldsymbol{Q'}}  
\end{bmatrix}
\end{align}
Here $\boldsymbol{Q} \in \mathbb{R}^{M \times T}$ is a matrix whose $(m, t)$-th element is defined as $q_t^m := \frac{1}{2} \boldsymbol{\alpha}_{t}' \boldsymbol{Y}_{t} \boldsymbol{K}_{t}^{m} \boldsymbol{Y}_{t} \boldsymbol{\alpha}_{t}$, 
where $\boldsymbol{1}_T$ is the $T$-dimensional all-one vector, and $\vec{*}$ is the matrix vectorization operator. As we show later, the matrix $\boldsymbol{A}$ is \ac{PSD}, and therefore, optimization problem \eqref{Optimize-Theta} is convex for which any quadratic problem solver can be employed to find the optimal $\boldsymbol{\theta}$ in each iteration.
%such as \ac {ADMM} \cite{boyd2011distributed}
The optimization problem \wrt\;$\boldsymbol{\hat{K}}$ is given as
%\begin{align}
%\min_{ \{\boldsymbol{\hat{K}_{t}} \}_{t}  \in \Psi_{\boldsymbol{\hat{K}}} } \eta\left(  - \boldsymbol{\theta}' \boldsymbol{b} + c \right)  +  \beta \left(   \boldsymbol{d}' \boldsymbol{A}^{-1} \boldsymbol{b} + \frac{1}{2} \boldsymbol{b}' \boldsymbol{A}^{-1} \boldsymbol{b} - 4c \right)
%\nonumber 
%\end{align}
%  \in \Psi_{\boldsymbol{\hat{K}}} 
\begin{align}
\label{Optimize-K_hat}
\min_{\boldsymbol{\hat{K}} } \left( \left(  \eta  - 4 \beta \right)  c + \frac{\beta}{2} \boldsymbol{b}' \boldsymbol{A}^{-1} \boldsymbol{b} \right) + 
\left( \beta \boldsymbol{d}' \boldsymbol{A}^{-1} - \eta \boldsymbol{\theta}'  \right) \boldsymbol{b} 
\end{align}
%The following proposition shows that this problem can be reduced to solving a simple \ac {QP}.
Using \propref{Prop-QP} (in the Appendix), this problem can be reduced to solving the following simple \ac {QP}
\begin{align}
\label{Optimize-V_hat}
\min_{\boldsymbol{\hat{v}}} \; \boldsymbol{\hat{v}}' \boldsymbol{\Sigma} \boldsymbol{\hat{v}}  
 - 2 \boldsymbol{\hat{v}}' \boldsymbol{a} 
\end{align} 
where $\boldsymbol{a} = \frac{1}{2} \left( \eta \boldsymbol{V} \boldsymbol{\theta} - \beta \boldsymbol{V} \boldsymbol{A}^{-1} \boldsymbol{d} \right)$, and $\boldsymbol{\Sigma} = \left( (\eta - 4 \beta) \boldsymbol{I}_{Tn^{2}} +  \frac{\beta}{2} \boldsymbol{P}\right)$ with $\boldsymbol{P}:=\boldsymbol{V} \boldsymbol{V}^{\dagger}$ and $\boldsymbol{V}^{\dagger}$ the Pseudo inverse of $\boldsymbol{V}$ defined in \eqref{eq:blockdVectorDef}. Also, $\boldsymbol{\hat{v}} : =[\boldsymbol{\hat{v}}_{1}, \ldots, \boldsymbol{\hat{v}}_{T}]' \in \mathbb{R}^{Tn^{2}}$, where $\boldsymbol{\hat{v}}_{t} := \vec{\boldsymbol{\hat{K}}_{t}} \in \mathbb{R}^{n^{2}}$. Note that the projection matrix $\boldsymbol{P}$ is \ac {PSD}. Therefore, the optimization problem \eqref{Optimize-V_hat} is convex in $\boldsymbol{\hat{v}}$ for  $\eta \geq  4 \beta $, and it has the well known analytical solution $\boldsymbol{\hat{v}} = \boldsymbol{\Sigma}^{-1} \boldsymbol{a}$.

%%%%%%%%%%%%%%%%%%%%%%%%%%%%%%%%%%%%%%%%%%%%%%%%%%%%%%%%%%%%%%%%%%%%%%%%%%%%%%%%
%%%%%%%%%%%%%%%%%%%%%%%%%%%%%%%%%%%%%%%%%%%%%%%%%%%%%%%%%%%%%%%%%%%%%%%%%%%%%%%%
%%%%%%%%%%%%%%%%%%%%%%%%%%%%%%%%%%%%%%%%%%%%%%%%%%%%%%%%%%%%%%%%%%%%%%%%%%%%%%%%
\section{Experiments}
\label{sec:Experiments}

In this section, we demonstrate the merits of \ac{MT-ONMKL} via a series of comparative experiments. 
%We examine the merits of learning our optimal neighborhood kernels (based on indefinite similarity matrices) over more traditional approach of choosing such kernels in an ad-hoc manner. 
%These models are mostly clustering based approaches which enable tasks to selectively share information with the related ones among all tasks and thus, prevent the effect of negative transfer.
%which, like ours, take advantage of task similarities by learning a shared feature representation across all the tasks.
%
%experimentally that the kernel minimizing generalization bound yields a good estimate of the neighbor-hood defining kernel, therefore renders improved generalization performance, compared to the approach of randomly chosen pre-specified kernel. Additionally, we demonstrate the capability of our proposed method by comparing with a few other classical \ac{MT-MKL} methods. Also, since \ac{ST} problem can be viewed as a special case of the \ac{MT} one with $T=1$, we also evaluate our method in \ac{ST} setting by comparing it with \cite{Liu2009} and a few \ac{MKL} methods. 
%
%\subsection{Experimental Setting}
\label{sec:exp_setting}
For all experiments, $1$ Linear, $1$ Polynomial with degree $2$, and $8$ Gaussian kernels with spread parameters $\left\{2^1,\ldots,2^{8}\right\}$ have been utilized as kernel functions for \ac{MKL}. All kernels are normalized as $k(\boldsymbol{x}_{1}, \boldsymbol{x}_{2}) \leftarrow k(\boldsymbol{x}_{1}, \boldsymbol{x}_{2})/\sqrt{k(\boldsymbol{x}_{1}, \boldsymbol{x}_{1}) k(\boldsymbol{x}_{2}, \boldsymbol{x}_{2})}$. Note that, in order to derive the need for \ac {MTL}, we intentionally keep the training set size small as only $20\%$ of the samples we use for each experiment. The rest of data are split in equal sizes for validation and testing. The SVM regularization parameter $C$ is chosen over the set $\left\{2^{-13},\ldots,2^{13}\right\}$; $\eta$ and $\beta$ are chosen over the set $\left\{1,2^1,2^2,\ldots,2^{40}\right\}$ via cross-validation. 
%Note that we intentionally kept the size of the training data small to drive the need for learning from other tasks, which diminishes as the training sets per task become large.
%
%to construct a \ac{SVM} based model with a linear combination of a set of base kernels in the way of \ac{MKL}. For this purpose, we conduct our experiments on 8 benchmark single task and 8 multi class data sets obtained from UCI repository \cite{Frank2010}. 
%
\subsection{Benchmark Datasets}
\label{sec:datasets}
We evaluate the performance of our method on the following datasets:

\textbf{Letter Recognition} dataset is a collection of handwritten words -- collected by Rob Kassel at MIT spoken Language System Group -- involves the eight tasks: `C' vs. `E', `G' vs. `Y', `M' vs. `N', `A' vs. `G', `I' vs. `J', `A' vs. `O', `F' vs. `T' and `H' vs. `N'. Each letter is represented by 8 by 16 pixel image, which forms a 128 dimensional feature vector.  We randomly chose 200 samples for each letter. An exception is letter ‘J’, for which only 189 samples were available.

\textbf{Landmine Detection} dataset consist of 29 binary classification tasks collected from various landmine fields. Each data sample is represented by a 9-dimensional feature vector extracted from radar images and is associated to a binary class label $y$. The feature vectors correspond to regions of landmine fields and include four moment-based features, three correlation-based features, one energy ratio feature, and one spatial variance feature. The objective is to recognize whether there is a landmine or not based on a region's features. 

\textbf{Spam Detection} dataset was obtained from ECML PAKDD 2006 Discovery challenge for the spam detection problem. For our experiments, we used Task B dataset which contains labeled training data (emails) from inboxes of $15$ different users. The goal is to construct a binary classifier for each user, detecting spam ($+$) emails from the non-spam ($-$) ones. Each email is represented by the term frequencies of the words resulting in $150K$ features from which we chose the $1000$ most frequent ones. 
%We also randomly chose $400$ emails per user inbox.

\textbf{SARCOS} dataset is generated from an inverse dynamics prediction system of a seven degrees-of-freedom (DOF) SARCOS anthropomorphic robot arm. This dataset consists of $28$ dimensions: the first $21$ dimensions are considered as features (including $7$ joint positions, $7$ joint velocities and $7$ joint accelerations), and the last 7 dimensions, corresponding to $7$ joint torques, are used as outputs. Therefore, there are $7$ tasks and the inputs are shared among all the tasks. for each $21$-dimensional observation, the goal is then to predict $7$ joint torques for the seven DOF. This dataset involves $48933$ observations from which we randomly sampled $2000$ examples for our experiments.

%\textbf{Sentiment Analysis} dataset contains Amazon products reviews on different domains, such as books, dvd ans so on. We choose $24$ domains (corresponding to $24$ tasks), and $100$ reviews per domain. Each domain is considered as a binary classification task with reviews labeled $(+)$ when rating $> 3$, and labeled $(-)$ when rating $<3$. Note that the reviews with rating $=3$ are excluded in this experiment as such sentiments were ambiguous and hard to predict, even with very large amount of data. Our features were defined using bag-of unigrams appearing at least $5$ times in all domains, yielding a dictionary of size $4000$. We then used a binary feature vector encoding the presence or absence of these frequent unigrams to define our instances.
\textbf{Short-term Electricity Load Forecasting} dataset which was released for the Global Energy Forecasting Competition (GEFComp2012). This dataset contains hourly-load history of a US utility in $20$ different zones from January $1$st, 2004 to December $31$, 2008. The goal is to predict the $1$-hour-ahead electricity load of these $20$ zones. For this purpose, we considered predictors consist of a delay vector of $8$ lagged hourly loads along with the calendar information including years, seasons, months, weekdays and holidays. Note that we normalized the data to unify the units of different features. Finally, we randomly sampled $2000$ (non-sequential) examples per each task for our experiments.
%\textbf{School} dataset which consists of examination scores of $15362$ students from $139$ secondary schools in London. Therefore, this dataset has $139$ tasks of predicting students performance for each school. Each sample is described by $27$ features including four school- and three student-specific attributes along with and an additional attribute for the bias term. 

\subsection{Experimental Results} 
\label{sec:exp_results}
 To assess the performance of our \ac{MT-ONMKL}, we compared it with some neighborhood kernel approaches reviewed in \sref{sec:NewModel}. In both cases, considering an \ac {SVM} formulation, the optimization over $\boldsymbol{\hat{K}}$ leads to the analytical solution $\boldsymbol{\hat{K}} = \boldsymbol{K} + (\boldsymbol{Y} \boldsymbol{\alpha})(\boldsymbol{Y} \boldsymbol{\alpha})'$ which depends on the labels of the training samples. More specifically, the first model in \cite{Liu2009} uses a Gaussian kernel matrix $\boldsymbol{K}$ with the spread parameter $\boldsymbol{\sigma}=4^{k-1}\sum_{i=1}^{n}\sum_{j=1}^{n}\left\|x_i-x_j\right\|_2^2/n^2$. Note that, as suggested by the authors in \cite{Liu2009}, we did cross-validation over $k={1,2,3}$ to choose the best kernel. We will be refereeing to this model as \ac{RPKL}. The second approach in \cite{Liu2013}, dubbed \ac {ONJKL}, instead of pre-specifying, it learns the kernel $\boldsymbol{K}$ in the form of a linear combination of a set of base kernels. Note that we modified the formulations in \cite{Liu2009,Liu2013} to \ac{MTL} setting. We also compare our approach with a simple \ac {KTA} model in which the neighborhood-defining matrix $\boldsymbol{\hat{K}}_{t}$ is fixed, and it is defined as $\boldsymbol{y}_{t} \boldsymbol{y}'_{t}$, for each task $t$.   

Moreover, we evaluate the performance of our model against the classical \ac {MT-MKL}, which considers $T$ inter-related \ac {SVM}-based formulations with multiple kernel functions, and jointly learns the parameter $\boldsymbol{\theta}_{t}$s of the kernels during the training process. AVeraged Multi-Task Multiple Kernel Learning (AVMTMKL), in which \ac{MKL} parameters are all fixed and set equal to $1/M$, is another method we consider in our comparison study. Finally, an \ac {ITL} model is used as a baseline, according to which each task is individually trained via a traditional single-task \ac {MKL} strategy, and the average performance over all tasks is taken to gauge the effectiveness of this method versus others. 
%The experimental results are reported in \tref{multi-task2}. According to these results, we observe that the proposed method outperforms the other models in most experiments. 
%
% To assess the performance of \ac{MT-ONMKL} in \ac{MTL}, we compared it to two simple variations of our framework and a popular \ac{MT-MKL} model \cite{Tang2009}, utilizing an \ac{SVM}-based formulation with partially-shared kernel function across tasks. The first variation of our framework dubbed \ac{RPK} uses Gaussian pre-specified kernels with  $\boldsymbol{\sigma}=4^{2}\sum_{i=1}^{n}\sum_{j=1}^{n}\left\|x_i-x_j\right\|_2^2/n^2$. The comparison against this variant will allow us to assess the utility of the optimal kernel choice of \eref{eq:Kthatstar} versus the more traditional approach of choosing such kernels in an ad-hoc manner. The second variant, dubbed AVMTMKL, differs from \ac{RPK} only in that the \ac{MKL} parameters are set all equal to $1/M$ and are not learned. We decided to include this variant, in an attempt to gauge the additional effect of \ac{MKL} to the performances of \ac{RPK} and \ac{MT-ONMKL}. To evaluate the effectiveness of our proposed \ac {MTL} method, we also compare it with a independent-task learning (\textbf{ITL}) approach according to which each task is individually trained via a traditional single-task strategy.

\begin{table}[h!]
\begin{center}
\caption{Experimental comparison between \ac{MT-ONMKL} and six other methods on five benchmark datasets. The superscript next to each model indicates its rank. The best performing algorithm gets rank of $1$.}
\label{multi-task2}
\tabcolsep=0.25cm
\begin{tabular}{l l l l l l}
\hline\noalign{\smallskip}
 \multicolumn{1}{c}{} & \multicolumn{3}{c}{Classification Accuracy}  &  \multicolumn{2}{c}{Regression MSE} \\
    \cmidrule{2-6}
%\toprule
     &         Landmine&          	Letter&  	            Spam&       	            SARCOS&     	        Load\\      
\midrule				
ITL$^{(5.6)}$	     &	$	60.01	^	{(5)}	$	&	$	87.75	^	{(7)}	$	&	$	94.66	^	{(6)}	$	&	$	25.44	^	{(6)}	$	&	$	8.14	^	{(4)}	$
\\
\midrule	

AVMTMKL$^{(5.1)}$   &	$	59.86	^	{(6)}	$	&	$	89.74	^	{(5)}	$	&	$	95.78	^	{(4.5)}	$	&	$	25.13	^	{(5)}	$	&	$	8.34	^	{(5)}	$
\\      
\midrule				

MT-MKL$^{(4.1)}$   &	$	59.11	^	{(7)}	$	   &	$	89.94	^	{(4)}	$	&	$	95.78	^	{(4.5)}	$	&	$	24.39	^	{(4)}	$	&	$	7.32	^	{(1)}	$
\\
\midrule	
				
RPKL$^{(4)}$       &	$	60.81	^	{(3)}	$   	&	$	90.63	^	{(2)}	$	&	$	94.59	^	{(7)}	$	&	$	15.14	^	{(2)}	$	&	$	8.42	^	{(6)}	$
\\
\midrule	
				
%RPKL(k=2)&	$	59.86	\pm	0.85	$	&	$	89.74	\pm	0.43	$	&	$	96.02	\pm	0.27	$	&	$	15.14	\pm	0.17	$	&	$	9.23	\pm	0.71	$
%\\
%\midrule	
%				
%RPKL(k=3)&	$	60.77	\pm	1.97	$	&	$	89.52	\pm	0.35	$	&	$	95.69	\pm	0.33	$	&	$	17.64	\pm	0.12	$	&	$	8.42	\pm	0.46	$
% \\
%\midrule	
			 	
ONJKL$^{(3.6)}$    &	$	60.36	^	{(4)}	$   	&	$	89.28	^	{(6)}	$	&	$	95.92	^	{(3)}	$	&	$	17.87	^	{(3)}	$	&	$	7.56	^	{(2)}	$
\\
\midrule

KTA$^{(4.2)}$      &	$	61.29	^	{(2)}	$	&	$	90.16	^	{(3)}	$	&	$	96.01	^	{(2)}	$	&	$	29.53	^	{(7)}	$	&	$	11.75	^	{(7)}	$
\\
\midrule	
				
\textbf{MT-ONMKL}$^{(1.4)}$&	$	62.61	^	{(1)}	$	&	$	91.91	^	{(1)}	$	&	$	97.62	^	{(1)}	$	&	$	13.2	^	{(1)}	$	&	$	7.87	^	{(3)}	$
\\
\bottomrule
\hline
\end{tabular}
\end{center}
\end{table}

\tref{multi-task2} reports the  average performance (accuracy for classification, and MSE for regression) over 20 runs of randomly sampled training sets for each experiment. The superscript next to each value indicates the rank of the corresponding model on the relevant data set, while the superscript next to each model name reflects its average rank over all data sets. Note  that we used Friedman's and Holm's post-hoc tests in \cite{Demvar2006}, using which a model can be statistically compared to a set of other methods over multiple data sets. According to this statistical analysis, we concluded that our model dominates all other methods at the significance level $\alpha = 0.05$.

Note that, although \ac {KTA} shows promising results in classification, it fails good results for regression problems. This is even more evident for SARCOS dataset which is considered a challenging problem due to the strong nonlinearity of the model caused by the extensive amount of superpositions of sine and cosine functions in robot dynamics \cite{vijayakumar2000locally}. This might bring one to the conclusion that, in complex prediction problems similar to robot inverse dynamic, using the output kernel might not be the best choice to align the optimal kernel.

\begin{figure}[htbp]
\centering
\subfigure[Classification (Letter)]{
   \includegraphics[width=0.48\textwidth]{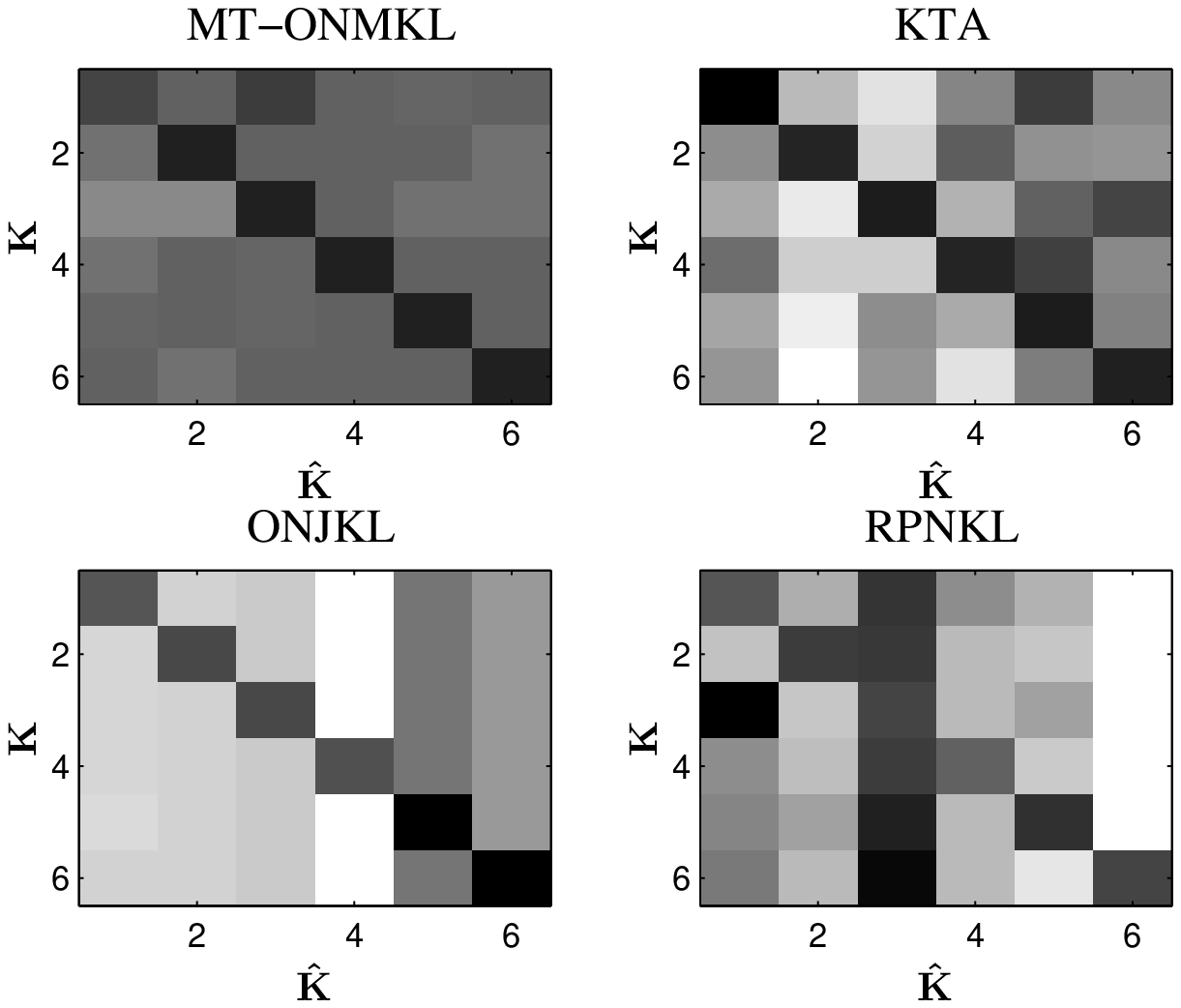}
    \label{fig:subfig1}
}
\subfigure[Regression (Load)]{
    \includegraphics[width=0.48\textwidth]{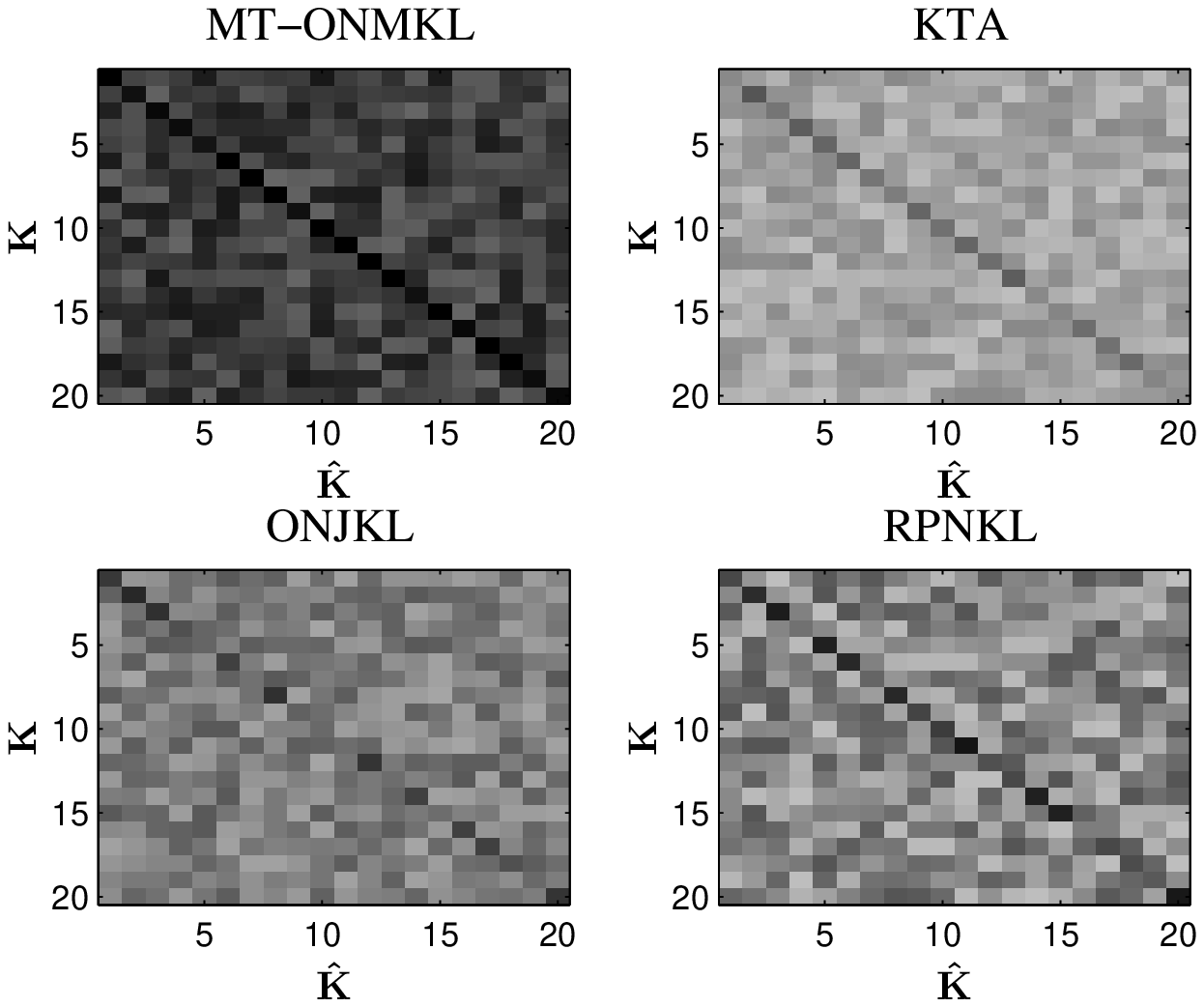}
    \label{fig:subfig2}
}

\caption{Kernel alignment between optimal and neighborhood matrices for each pair of tasks}
\label{fig:Feature space}
\end{figure}
For all four alignment-regularized models, the pairwise alignments between the optimal kernels and the neighborhood-defining matrices are shown in \fref{fig:Feature space}. As one can observe, for both classification and regression problems Letter and Load, our optimal kernel for each task $t$ is highly aligned, not only with its own corresponding neighborhood kernel, but also with the neighborhood kernels of other tasks. This would suggest that our model can provide best alignment as well as best performance of the final kernels among all other alignment-based models considered in this study.

\section{Conclusions}
\label{sec:Conclusions}
%\vspace{-1.5mm}
In this work, we proposed a novel \ac {SVM}-based \ac{MT-MKL} framework for both classification and regression. Our new algorithm improves over the existing kernel-based methods, which have been demonstrated to be good performers on a variety of classification and regression tasks in the past. Our model, particularly, learns an optimal kernel simultaneously with a neighborhood (possibly indefinite) kernel, based on a Rademacher complexity-regularized model. As opposed to the previous approaches, our \ac{MT-ONMKL} model identifies the neighborhood defining kernels in a much more principled manner. In specific, they are chosen as the ones minimizing the Rademacher complexity bound of alignment-regularized models. The performance advantages reported for both classification and regression problems largely seem to justify the arguments related to the introduction of this new model.

% Remove if you don't have appendix.
\newpage
\appendix
% Reset all acronyms
\acresetall
%\begin{figure}
%\centering
%\includegraphics[width=0.7\linewidth]{./ECML2014paperA}
%\caption{}
%\label{fig:ECML2014paperA}
%\end{figure}
\newpage

%%%%%%%%%%%%%%%%%%%%%%%%%%%%%%%%%%%%%%%%%%%%%%%%%%%%%%%%%%%%%%%%%%%%%%%%%%%%%%%%
%%%%%%%%%%%%%%%%%%%%%%%%%%%%%%%%%%%%%%%%%%%%%%%%%%%%%%%%%%%%%%%%%%%%%%%%%%%%%%%%
%%%%%%%%%%%%%%%%%%%%%%%%%%%%%%%%%%%%%%%%%%%%%%%%%%%%%%%%%%%%%%%%%%%%%%%%%%%%%%%%
\section*{Supplementary Materials}

A useful lemma in deriving the generalization bound of \thrmref{thrm:thrm} is provided next.

\begin{lemm}
\label{lemm:lemma}
Let $\mathbf{A}, \mathbf{B} \in \mathbb{R}^{N \times N}$ and let $\boldsymbol{\sigma} \in \mathbb{R}^N$ be a vector of independent Rademacher random variables. Let $\circ$ denote the Hadamard (component-wise) matrix product. Then, it holds that

\begin{align}
	\mathbb{E}_{\boldsymbol{\sigma}} & \left\{ \left( \boldsymbol{\sigma}' \mathbf{A} \boldsymbol{\sigma} \right)    \left( \boldsymbol{\sigma}' \mathbf{B} \boldsymbol{\sigma} \right) \right\}  = \trace{\mathbf{A}}  \trace{\mathbf{B}}  + \nonumber 
	\\
	&  + 2 \left( \trace{\mathbf{A}\mathbf{B}} - \trace{\mathbf{A} \circ \mathbf{B}}  \right)
\end{align}

\end{lemm}
\begin{proof}

Let $[\cdot]$ denote the Iverson bracket, such that $[\mathrm{predicate}] = 1$, if  $\mathrm{predicate}$ is true and $0$, if false. The expectation in question can be written as

\begin{align}
	\label{eq:app1}
	& \Expect[\boldsymbol{\sigma}]{  \left( \boldsymbol{\sigma}' \mathbf{A} \boldsymbol{\sigma} \right)    \left( \boldsymbol{\sigma}' \mathbf{B} \boldsymbol{\sigma} \right) } =  \sum_{i,j,k,l} a_{i,j} b_{k,l} \Expect{ \sigma_i \sigma_j \sigma_k \sigma_l}
\end{align}

where the indices of the last sum run over the set $\left\{1, \ldots, N\right\}$. Since the components of $\boldsymbol{\sigma}$ are independent Rademacher random variables, it is not difficult to verify the fact that $\Expect{ \sigma_i \sigma_j \sigma_k \sigma_l} = 1$ only in the following four cases: $\left\{i=k, j=l, i \neq l \right\}$, $\left\{ i=j, k=l, i \neq k \right\}$, $\left\{ i=l, k=j, i \neq k \right\}$ and $\left\{ i=j, j=k, k=l \right\}$; in all other cases, $\Expect{ \sigma_i \sigma_j \sigma_k \sigma_l} = 0$. Therefore, it holds that

%\begin{align}
%	& \Expect{ \sigma_i \sigma_j \sigma_k \sigma_l} = [i=k][j=l][i \neq l] + [i=j][k=l][i \neq k] 
%	\nonumber
%	\\
%	\label{eq:app2}
%	& + [i=l][k=j][i \neq k] +  [i=j][j=k][k=l]
%\end{align}

\begin{equation} \label{eq:app2}
\begin{split}
	 \Expect{ \sigma_i \sigma_j \sigma_k \sigma_l} & = [i=k][j=l][i \neq l] + [i=j][k=l][i \neq k] \\
	 & + [i=l][k=j][i \neq k] +  [i=j][j=k][k=l]
\end{split}
\end{equation}

Substituting \eref{eq:app2} into \eref{eq:app1}, after some algebraic operations, yields the desired result.

\end{proof}

%%%%%%%%%%%%%%%%%%%%%%%%%%%%%%%%%%%%%%%%%%%%%%%%%%%%%%%%%%%%%%%%%%%%%%%%%%%%%%%%
%%%%%%%%%%%%%%%%%%%%%%%%%%%%%%%%%%%%%%%%%%%%%%%%%%%%%%%%%%%%%%%%%%%%%%%%%%%%%%%%
%%%%%%%%%%%%%%%%%%%%%%%%%%%%%%%%%%%%%%%%%%%%%%%%%%%%%%%%%%%%%%%%%%%%%%%%%%%%%%%%
\section*{Proof of \thrmref{thrm:thrm}}
\label{sec:ProofTheorem}
%Using Theorems $16$ and $17$ of \cite{Maurer2006}, one can show that, given a multi-task \ac{HS} $\mathcal{F}$ like ours, for all $f \in \mathcal{F}$, for $\delta > 0$ and for fixed $\gamma > 0$, with probability at least $1-\delta$ the following holds
%
%
%\begin{align}
%	\label{eq:thm:1}
%	er\left( f \right) \leq \widehat{er}_{\gamma}\left( f \right) + \frac{2}{\gamma} \hat{R}\left( \mathcal{F} \right) + \sqrt{ \frac{9 \ln \frac{2}{\delta}}{2 T N}}
%\end{align}
%
%
%\noindent
%where the \ac{ERC} $\hat{R}$ of $\mathcal{F}$ is given as 
%
As mentioned earlier, the Rademacher complexity of function class $\mathcal{F}$ is defined as
\begin{align}
	\label{eq:thm:2}
	\mathfrak{\hat{R}}(\mathcal{F}) = \frac{1}{nT} \Expect[\boldsymbol{\sigma}]{  \underset{f \in \mathcal{F} }{\sup}  \sum_{t=1}^{T} \sum_{i=1}^{n} \sigma_t^i  f_t(x_t^i) }
\end{align}
\noindent
where $\sigma_t^i$'s are i.i.d.\ Rademacher random v ariables. By invoking the Representer Theorem (e.g. see \cite{Scholkopf2001a}), \eqref{eq:thm:2} becomes
\begin{align}
	\label{eq:thm:3}
	\mathfrak{\hat{R}}(\mathcal{F}) = \frac{1}{nT} \Expect[\boldsymbol{\sigma}]{  \underset{\underset{t = 1, \ldots, T}{\boldsymbol{\alpha}_t \in \Psi_{\boldsymbol{\alpha}}, \boldsymbol{\theta} \in \Psi_{\theta} }}{\sup}  \sum_{t=1}^{T} \boldsymbol{\sigma}_t'  \boldsymbol{K}_t \boldsymbol{\alpha}_t }
\end{align}
\noindent
where $\Psi_{\theta} := \left\{ \boldsymbol{\theta} \in \mathbb{R}^{M+MT}  \ : \ \boldsymbol{\theta} \succeq \mathbf{0}, \   \boldsymbol{\theta}' \boldsymbol{A} \boldsymbol{\theta} -  \boldsymbol{\theta}' \boldsymbol{b} + c \leq \rho \right\}$. Also,  $\Psi_{\boldsymbol{\alpha}}$ is defined as
$$\Psi_{\boldsymbol{\alpha}} := \left\{ \boldsymbol{\alpha} = \left( \boldsymbol{\alpha}_{1} \ldots, \boldsymbol{\alpha}_{T} \right): \forall t,\;  \boldsymbol{\alpha}_t \in \mathbb{R}^n \ : \ \sum_{t=1}^T \boldsymbol{\alpha}_t^T \boldsymbol{K}_t \boldsymbol{\alpha}_t \leq R^2 \right\}.$$

 Instead, consider the relaxed constraint set $\Psi_{\boldsymbol{\alpha}_t}:= \left\{ \boldsymbol{\alpha}_t \in \mathbb{R}^n  \ : \ \boldsymbol{\alpha}_t^T \boldsymbol{K}_t \boldsymbol{\alpha}_t \leq R^2 \right\}$. Then, it follows that
\begin{align}
	\label{eq:thm:4}
	\mathfrak{\hat{R}}(\mathcal{F}) \leq \frac{1}{nT} \Expect[\boldsymbol{\sigma}]{  \underset{\underset{t = 1, \ldots, T}{\boldsymbol{\alpha}_t \in \Psi_{\boldsymbol{\alpha}_t}, \boldsymbol{\theta} \in \Psi_{\theta} }}{\sup}  \sum_{t=1}^{T} \boldsymbol{\sigma}_t'  \boldsymbol{K}_t \boldsymbol{\alpha}_t }
\end{align}
\noindent
Using $\boldsymbol{K}_t := \sum_{m = 1}^{M} \theta_t^m \boldsymbol{K}_t^m$, $\theta_t^m = \mu^m + \lambda_t^m$, and $\boldsymbol{\theta}_{t} := \boldsymbol{\mu} +  \boldsymbol{\lambda}_{t} \in \mathbb{R}^M$, the right-hand side of \eref{eq:thm:4}, if first optimized w.r.t. the $\boldsymbol{\alpha}_t$'s, yields
\begin{align}
	\label{eq:thm:5}
	\mathfrak{\hat{R}}(\mathcal{F}) \leq \frac{\sqrt{R}}{nT} \Expect[\boldsymbol{\sigma}]{  \underset{\boldsymbol{\theta} \in \Psi_{\theta} }{\sup}  \sum_{t=1}^{T} \sqrt{ \boldsymbol{\theta}'_{t}  \mathbf{u}_t } } 
\end{align}
\noindent
where $\boldsymbol{\mu} := [\mu^1, \cdots, \mu^M]' \in \mathbb{R}^M$,  $\boldsymbol{\lambda}_t := [\lambda_t^1, \cdots, \lambda_t^M]'$, $\boldsymbol{u}_t : \left[ \boldsymbol{\sigma}'_t \boldsymbol{K}_t^1 \boldsymbol{\sigma}_t \ldots \boldsymbol{\sigma}'_t \boldsymbol{K}_t^M \boldsymbol{\sigma}_t \right]$. Also, using $l_p-to-l_q$ conversion, we have for any non-negative vector $ \boldsymbol{a}_1$ and any $0 \leq q  \leq  p  \leq  \infty$:
\begin{align}
 \left\| \boldsymbol{a}_1 \right\|_q  = \left\langle \boldsymbol{1}, \boldsymbol{a}_1^{q} \right\rangle^{\frac{1}{q}} \stackrel{\text{H\"{o}lder's}}{\leq} \left( \left\| \boldsymbol{1} \right\|_{(p/q)^{*}}  \left\| \boldsymbol{a}_1^{q} \right\|_{(p/q)}  \right) ^{\frac{1}{q}} = T^{\frac{1}{q} - \frac{1}{p}}  \left\| \boldsymbol{a}_1 \right\|_p
 \nonumber
 \end{align}
Taking $q  = 1/2$ and $p=1$, it can be shown that $\sum_{t=1}^{T} \sqrt{ \boldsymbol{\theta}'_{t}  \boldsymbol{u}_t } \leq \sqrt{ T \boldsymbol{\theta}' \boldsymbol{u}}$, where 
\noindent
\begin{align}
\label{eq:blockuVectorDef}
\mathbf{u} := \begin{bmatrix}
\boldsymbol{U} \boldsymbol{1}_T \\ 
\vec{\boldsymbol{U'}}  
\end{bmatrix} \in \mathbb{R}^{M+MT}
\end{align}
\noindent
and $\boldsymbol{U} \in \mathbb{R}^{M \times T}$, whose $(m, t)$-th element is defined as $u_t^m := \boldsymbol{\sigma}'_t \boldsymbol{K}_t^m \boldsymbol{\sigma}_t$, \erefequ{eq:thm:5} becomes
%\begin{align}
	%\label{eq:thm:5}
	%\hat{R}(\mathcal{F}) \leq \frac{2 \sqrt{R}}{NT} \Expect[\boldsymbol{\sigma}]{  \underset{\boldsymbol{\theta} \in \Psi_{\theta} }{\sup}  \sum_{t=1}^{T} \sqrt{\boldsymbol{\sigma}_t'  \mathbf{K}_t \boldsymbol{\sigma}_t }}
%\end{align}
\begin{align}
	\label{eq:thm:6}
	\mathfrak{\hat{R}}(\mathcal{F}) \leq \frac{1}{n} \sqrt{ \frac{R}{T} } \Expect[\boldsymbol{\sigma}]{  \underset{\boldsymbol{\theta} \in \Psi_{\theta} }{\sup} \sqrt{ \boldsymbol{\theta}'  \boldsymbol{u} } } 
\end{align}
\noindent
Optimizing w.r.t. $\boldsymbol{\theta}$ finally yields
\begin{align}
	\label{eq:thm:7}
	\mathfrak{\hat{R}}(\mathcal{F})  \leq \frac{1}{n} \sqrt{ \frac{R}{2 T} } \Etmp_{\boldsymbol{\sigma}} \left\{ 
	\left( \boldsymbol{u}' \boldsymbol{A}^{-1} \boldsymbol{b} + \sqrt{\boldsymbol{u}' \boldsymbol{A}^{-1} \boldsymbol{u}}  \sqrt{ \boldsymbol{b}' \boldsymbol{A}^{-1} \boldsymbol{b} + 4(\rho - c) } \right)^{\frac{1}{2}} \right\} 
\end{align}
\noindent
By applying Jensen's Inequality twice, we obtain
\begin{align}
	\label{eq:thm:8}
	\mathfrak{\hat{R}}(\mathcal{F})   \leq \frac{1}{n} \sqrt{ \frac{R}{2 T} }  
	\left( \Expect[\boldsymbol{\sigma}]{ \boldsymbol{u}' \boldsymbol{A}^{-1} \boldsymbol{b} } + \sqrt{ \Expect[\boldsymbol{\sigma}]{ \boldsymbol{u}' \boldsymbol{A}^{-1} \boldsymbol{u}} } \sqrt{ \boldsymbol{b}' \boldsymbol{A}^{-1} \boldsymbol{b} + 4(\rho - c) } \right)^{\frac{1}{2}} 
\end{align}
\noindent
If $\boldsymbol{d}$ is defined as \eref{eq:blockdVectorDef}, the first expectation evaluates to 
\begin{align}
	\label{eq:thm:9}
	\Expect[\boldsymbol{\sigma}]{ \boldsymbol{u}' \boldsymbol{A}^{-1} \boldsymbol{b} } = \boldsymbol{d}' \boldsymbol{A}^{-1} \boldsymbol{b} 
\end{align}
\noindent
%since $\Expect[\boldsymbol{\sigma}]{ \boldsymbol{\sigma}' \mathbf{K} \boldsymbol{\sigma} } = \trace{\mathbf{K}}$, for any square matrix $\mathbf{K}$. Now, let $\vec{.}$ stack all the columns of a matrix into a vector. Define the block matrix $\boldsymbol{V}_t \in \mathbb{R}^{n^2 \times (M+MT)}$ of the following form
% \begin{align}
%\label{eq:vtDef}
%	\boldsymbol{V}_t \triangleq \begin{bmatrix}
%	\tilde{\boldsymbol{V}}_t & \tilde{\boldsymbol{V}}_t \bigotimes \boldsymbol{e}_t'
%	\end{bmatrix}
%\end{align}
%where $\tilde{\boldsymbol{V}}_t := \left[\vec{\boldsymbol{K}_t^1} \cdots \vec{\boldsymbol{K}_t^M} \right] \in \mathbb{R}^{n^2 \times M}$, $\boldsymbol{e}_t \in \mathbb{R}^T$ is a vector whose $t$-th element is $1$ and other elements are $0$, and $\bigotimes$ stands for the Kronecker product. Define new matrix $\boldsymbol{V}\in \mathbb{R}^{Tn^{2} \times (M+MT)}$ as $\boldsymbol{V} : = [\boldsymbol{V}_{1}, \ldots, \boldsymbol{V}_{T}]'$. 
Note that with the aid of \lemmref{lemm:lemma}, and definition of $\boldsymbol{V}$ in \eref{eq:blockdVectorDef}, it can be shown that
\begin{align}
	\label{eq:thm:10}
	\Expect[\boldsymbol{\sigma}]{ \boldsymbol{u}' \boldsymbol{A}^{-1} \boldsymbol{u} } \leq \boldsymbol{d}' \boldsymbol{A}^{-1} \boldsymbol{d} +2 \trace{\boldsymbol{V} \boldsymbol{A}^{-1} \boldsymbol{V}'}
\end{align}
\noindent
Combining \eref{eq:thm:8}, \eref{eq:thm:9} and \eref{eq:thm:10} we conclude that
\begin{align}
	\label{eq:thm:11}
	\mathfrak{\hat{R}} \left( \mathcal{F} \right) & \leq \frac{1}{n} \sqrt{ \frac{R}{2 T} }  
	\left( \boldsymbol{d}' \boldsymbol{A}^{-1} \boldsymbol{b}  + \sqrt{ \boldsymbol{d}' \boldsymbol{A}^{-1} \boldsymbol{d} + 2 \trace{\boldsymbol{V} \boldsymbol{A}^{-1} \boldsymbol{V}'}} \sqrt{ \boldsymbol{b}' \boldsymbol{A}^{-1} \boldsymbol{b} + 4(\rho - c) } \right)^{\frac{1}{2}} 
  \nonumber\\
  &\leq \frac{1}{n} \sqrt{ \frac{R}{2 T} }  
  	\left( \boldsymbol{d}' \boldsymbol{A}^{-1} \boldsymbol{b}  + \frac{1} {2} \left[\left( \boldsymbol{d}' \boldsymbol{A}^{-1} \boldsymbol{d} + 2 \trace{\boldsymbol{V} \boldsymbol{A}^{-1} \boldsymbol{V}'} \right) + \left( \boldsymbol{b}' \boldsymbol{A}^{-1} \boldsymbol{b} + 4(\rho - c) \right) \right] \right)^{\frac{1}{2}} \nonumber
\end{align}
\noindent
where we used the Arithmetic-Geometric Mean inequality in the last step.

%%%%%%%%%%%%%%%%%%%%%%%%%%%%%%%%%%%%%%%%%%%%%%%%%%%%%%%%%%%%%%%%%%%%%%%%%%%%%%%
%%%%%%%%%%%%%%%%%%%%%%%%%%%%%%%%%%%%%%%%%%%%%%%%%%%%%%%%%%%%%%%%%%%%%%%%%%%%%%%
%%%%%%%%%%%%%%%%%%%%%%%%%%%%%%%%%%%%%%%%%%%%%%%%%%%%%%%%%%%%%%%%%%%%%%%%%%%%%%%
\begin{prop}
\label{Prop-QP}
Let $\vec{.}$ stack all the columns of a matrix into a vector. Define the matrix $\boldsymbol{V} \in \mathbb{R}^{Tn^2 \times (M+MT)}$ as \eqref{eq:blockdVectorDef}.
. Let $\boldsymbol{V}^{\dagger}$ denote the Pseudo inverse of matrix $\boldsymbol{V}$ , and define the orthogonal projection $\boldsymbol{P}:=\boldsymbol{V} \boldsymbol{V}^{\dagger}$. Also, let the vector $\boldsymbol{\hat{v}} : =[\boldsymbol{\hat{v}}_{1}, \ldots, \boldsymbol{\hat{v}}_{T}]' \in \mathbb{R}^{Tn^{2}}$, where $\boldsymbol{\hat{v}}_{t} := \vec{\boldsymbol{\hat{K}}_{t}} \in \mathbb{R}^{n^{2}}$. If one defines $\boldsymbol{\Sigma} = \left( (\eta - 4 \beta) \boldsymbol{I}_{Tn^{2}} +  \frac{\beta}{2} \boldsymbol{P}\right)$, and $\boldsymbol{a}= \frac{1}{2} \left( \eta \boldsymbol{V} \boldsymbol{\theta} - \beta \boldsymbol{V} \boldsymbol{A}^{-1} \boldsymbol{d} \right)$, then the solution $\boldsymbol{\hat{v}}^{*}$ of the following \ac {QP} 
\begin{align}
\min_{\boldsymbol{\hat{v}}} \; \boldsymbol{\hat{v}}' \boldsymbol{\Sigma} \boldsymbol{\hat{v}} - 2 \boldsymbol{\hat{v}}' \boldsymbol{a}
\end{align} 
is the same as the solution of the optimization problem \eqref{Optimize-K_hat}. 
\end{prop}
\begin{proof}
The proof follows from the fact that $\boldsymbol{A} = \boldsymbol{V}' \boldsymbol{V}$, $\boldsymbol{b}  = \boldsymbol{V}' \boldsymbol{\hat{v}}$ and $c = \boldsymbol{\hat{v}}' \boldsymbol{\hat{v}} $. Replacing these quantities in \eqref{Optimize-K_hat} completes the proof.
\end{proof}

\bibliographystyle{plainurl}
\bibliography{Arxiv2017paperA}
% The plainrul style will printout the URL field of each bib entry
% and hyperref will create a clickable link.
%\bibliographystyle{plainurl}
% Modify the bibliography file name as necessary.

\end{document}